\documentclass[journal]{IEEEtran}
\usepackage{amsmath,amsthm,amssymb,amsfonts,mathrsfs}
\usepackage{algorithmic,algorithm}
\usepackage{leftidx}
\usepackage{array}
\usepackage{textcomp}
\usepackage{stfloats}
\usepackage{url}
\usepackage{verbatim}
\usepackage{graphicx}
\usepackage{cite}
\usepackage{bm}
\usepackage{subfigure} 

\usepackage{soul, color, xcolor}
\soulregister{\cite}7 
\soulregister{\citep}7 
\soulregister{\citet}7 
\soulregister{\ref}7 
\soulregister{\pageref}7 

\theoremstyle{definition}
\newtheorem{theorem}{\textbf{Theorem}}

\newtheorem{lemma}{\textbf{Lemma}}

\usepackage{makecell}
\usepackage{graphicx}
\usepackage{textcomp}
\usepackage{array} 
\usepackage{longtable}
\usepackage{booktabs}
\usepackage{float}

\begin{document}

\title{Neural Predictor for Flight Control with Payload}

\author{Ao Jin$^{\dagger}$, %~\IEEEmembership{Student Member,~IEEE,}
	Chenhao Li$^{\dagger}$,
	Qinyi Wang,
	Ya Liu,
	Panfeng Huang,~\IEEEmembership{Senior Member,~IEEE},
	Fan Zhang$^*$,~\IEEEmembership{Member,~IEEE} 
	
	\thanks{Ao Jin is with Research Center for Intelligent Robotics, School of Automation, Northwestern Polytechnical University, Xi’an 710072, China. Chenhao Li, Qinyi Wang, Ya Liu, Panfeng Huang and Fan Zhang are with Research Center for Intelligent Robotics, Shaanxi Province Innovation Team of Intelligent Robotic Technology, School of Astronautics, Northwestern Polytechnical University, Xi’an 710072, China (E-mail: jinao@mail.nwpu.edu.cn, fzhang@nwpu.edu.cn).}% <-this % stops a space
	\thanks{$\dagger$ These authors contributed equally to this work. }% <-this % stops a space
	}

        % <-this % stops a space
%\thanks{This paper was produced by the IEEE Publication Technology Group. They are in Piscataway, NJ.}% <-this % stops a space
%\thanks{Manuscript received April 19, 2021; revised August 16, 2021.}}

%\markboth{IEEE ROBOTICS AND AUTOMATION LETTERS}%
%{Shell \MakeLowercase{\textit{et al.}}: Neural Predictor for Flight Control with Payload}


\maketitle

\begin{abstract}
Aerial robotics for transporting suspended payloads as the form of freely-floating manipulator are growing great interest in recent years. However, the force/torque caused by payload and residual dynamics will introduce unmodeled perturbations to the aerial robotics, which negatively affects the closed-loop performance. Different from estimation-like methods, this paper proposes Neural Predictor, a learning-based approach to model force/torque induced by payload and residual dynamics as a dynamical system. It yields a hybrid model that combines the first-principles dynamics with the learned dynamics. The hybrid model is then integrated into a MPC framework to improve closed-loop performance. Effectiveness of proposed framework is verified extensively in both numerical simulations and real-world flight experiments. The results indicate that our approach can capture force/torque caused by suspended payload and residual dynamics accurately, respond quickly to the changes of them and improve the closed-loop performance significantly. In particular, Neural Predictor outperforms a state-of-the-art learning-based estimator and has reduced the force and torque estimation errors by up to 66.15\% and 33.33\% while requiring less samples. The code of proposed Neural Predictor can be found at {\url{https://github.com/NPU-RCIR/Neural-Predictor.git}}.
\end{abstract}

\begin{IEEEkeywords}
Learning Based Control, Model Learning for Control, Model Predictive Control.
\end{IEEEkeywords}

\section{INTRODUCTION}
Owing to the advantages in terms of strong flexibility, safe manipulation, low cost and transportability, tethered-UAV system has been widely leveraged in various aerial transportation tasks. By attaching the payload to UAV through flexible tethers, there is no need to consider the shape and size of payload, making it highly adaptable to different transportation scenarios.

While attaching payloads to UAVs using flexible tethers offers significant advantages over rigid attachment methods such as robotic arms or clamps, several challenges remain to be addressed for real-world high-precision transportation tasks. The primary challenge in tethered-UAV systems lies in accurately estimating or predicting the external force/torque vector induced by the tether and payload. Recently, machine learning methods have demonstrated remarkable success in estimating external uncertainties and designing controllers across diverse applications. Additionally, the data-driven Koopman operator (DDKO) theory, which enables the discovery of unknown system dynamics from data in an analytical framework,  has garnered great attentions in control community. 

In this work, we introduce a learning-based framework called \textbf{Neural Predictor} to model the force/torque induced by the payload and residual dynamics of a tethered-UAV system. Specifically, leveraging insights from DDKO theory, we derive the formulation of lifted linear system (LLS), which is subsequently learned from data to capture the external force/torque. By integrating the learned dynamics with the nominal system dynamics, the hybrid model of the tethered-UAV system is constructed. This hybrid model is then embedding within a model predictive control framework, referred to as \textbf{NP-MPC}. An overview of the proposed framework is illustrated in Fig. \ref{illustration_framework}. Through extensive simulations and real-world flight experiments, we demonstrate that our approach not only outperforms state-of-the-art learning-based estimators in predicting external forces and torques but also achieves significant improvements in closed-loop control performance.

\section{Related Work}
The modeling of tethered-UAV system has been studied a lot in the literature. The studies \cite{Sreenath2013b,Cruz2015,Tang2018a,Li2023a} focused on modeling the UAV, tether and payload as a whole system. There is a transition of established dynamics when tether tension becomes zero, which posed great challenges for controller design. In addition, the tether is assumed to be massless and straight in \cite{Sreenath2013b}. However, some of them modeled tethered-UAV system upon assumptions that hardly satisfy in practice, some of them require the prior knowledge of payload and tether for modeling. The studies \cite{Yu2023c,Kong2024} estimated external force caused by unknown payload by adaptive law within the framework of backstepping, but the external torque did not consider. 

Recently, machine learning schemes have emerged as a promising approach for modeling complex aerodynamics and external force/torque of quadrotor. NeuroBEM \cite{Bauersfeld2021} utilized deep neural networks and BEM theory to model aerodynamic effect during high speeds or agile maneuvers. Neural-Fly \cite{OConnell2022} adopts a deep neural network (DNN) to predict the effect of strong wind. Gaussian processes (GP) regression \cite{Torrente2021} is also used to model residual aerodynamics but it suffers from computational complexity problem. The study \cite{Chee2022,Jiahao2023,Duong2024} use neural ordinary differential equations (NODE) to model uncertainty and integrate them into controller. However, the learned dynamics in these works are black-box and implicit. This leads to hindrances in analyzing the dynamical characteristics of external uncertainty and residual dynamics. Moreover, the learning-based models proposed in these works do not provide theoretical guarantees of the boundedness of prediction error either. The work \cite{Wang2024e} proposed NeuroMHE framework to estimate residual aerodynamics in a moving horizon manner. But it needs to solve a optimization problem online and the selection of moving horizon is thorny in practice.

\begin{figure*}[!t]
	\centering
	\includegraphics[width=38pc]{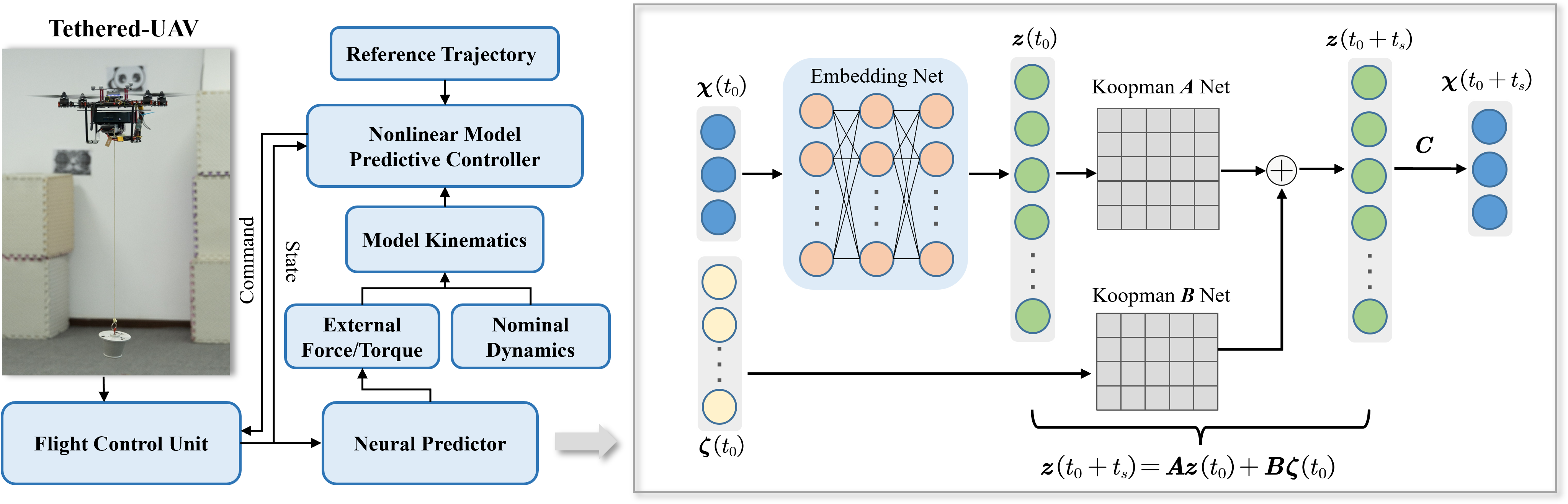} 
	\caption{The illustration of proposed framework: Neural Predictor cooperated with NMPC scheme.} \label{illustration_framework}
\end{figure*}

The \textbf{main contributions} of our work are summarized as follows. First, inspired by DDKO theory, the external force/torque induced by payload and residual dynamics during flight are explicitly modeled via a linear system, with the model parameters derived through a data-driven learning methodology. It results our proposed \textbf{Neural Predictor}, which offers greater interpretability than traditional black-box models. Furthermore, a theoretical guarantee is given for the boundedness of prediction error. Second, a learning-based control framework \textbf{NP-MPC} shown in Fig. \ref{illustration_framework} is proposed for flight control with payload. Our framework eliminates the necessity for modeling of payload and does not rely on any force/torque sensors. The results in both simulations and real-world experiments indicate that our proposed framework not only owns better ability of fast response and accurate estimation of external force/torque against state-the-of-art estimator, but also improves the closed-loop performance significantly.

\section{Problem Statement}
It is assumed that the quadrotor is a 6 degree-of-freedom (DoF) rigid body with mass $m$ and moment of inertia $\bm{J}\in \mathbb{R}^{3\times3}$. Then the dynamics of quadrotor of tethered UAV system could be written as 
\begin{equation}
	\begin{aligned}
		&\dot{\boldsymbol{p}}_w=\boldsymbol{v}_w, \ 
		\dot{\boldsymbol{v}}_w=m^{-1}\left(-mg\bm{z}+\boldsymbol{R}f_u\bm{z}+\bm{f}_p+\bm{f}_{\text{res}}\right)\\
		&\dot{\bm{R}}=\bm R \bm \omega_b^{\times }, \ 
		\dot{\boldsymbol{\omega}_b}=\boldsymbol{J}^{-1}\left(-\bm \omega_b^{\times}\bm J \bm \omega_b+\boldsymbol{\tau}_m+ \bm{\tau}_p+ \bm{\tau}_{\text{res}}\right) 
	\end{aligned}\label{quadrotordynamics}
\end{equation}
where $\bm{f}_p$ and $\bm{\tau}_p$ are the force and torque subjected to the quadrotor caused by payload, $\bm{f}_{\text{res}}$ and $\bm{\tau}_{\text{res}}$ are residual forces and torques, respectively. And $\bm{p}_w=\left[x,y,z\right]^\mathrm{T}$ and $\bm{v}_w=\left[v_x,v_y,v_z\right]^T$ are the position and velocity of the quadrotor's center-of-mass in the world frame $\mathcal{I}$, $\boldsymbol{\omega}_b=\left[{\omega}_x,{\omega}_y,{\omega}_z\right]^\mathrm{T}$ is the angular velocity in  the quadrotor body frame $\mathcal{B}$. The $\boldsymbol{R}\in SO^3$ is the rotation matrix from $\mathcal{B}$ and $\mathcal{I}$. The gravitational acceleration is $g$ and $\bm{z}=\left[0,0,1\right]^\mathrm{T}$. The $f_u$ and $\boldsymbol{\tau}_m=\left[{\tau}_{mx},{\tau}_{my},{\tau}_{mz}\right]^\mathrm{T}$ denote the total thrust and torque produced by the quadrotor’s four motors. We define $\bm{x}=\left[\bm{p}_w^\mathrm{T},\bm{v}_w^\mathrm{T},vec(\bm R),\bm{\omega}^\mathrm{T}\right]^\mathrm{T}$ as the quadrotor state, where $vec(\cdot)$ represents the vectorization of matrix and $\bm{u} = [{f}_u, {\tau}_{mx},{\tau}_{my},{\tau}_{mz}]^\mathrm{T}$ as the control input, respectively. For the sake of completeness, we define the external force/torque as the sum of force/torque caused by payload and residual force/torque, i.e., $\bm f_e = \bm f_p+ \bm f_{\text{res}}$ and $\bm \tau_e = \bm \tau_p+\bm \tau_{\text{res}}$.

\textit{Problem Formulation}: Considering the tethered-UAV system dynamics (\ref{quadrotordynamics}), the objective of this work is to learn the external force $\bm{f}_e$ and $\bm{\tau}_e$ induced by suspended payload and residual dynamics from data such that the closed-loop performance can be improved.

\section{Learning External Force and Torque Using Lifted Linear System}
\label{s3}
This section elaborates the details of proposed framework. The main idea is modeling the external force $\bm{f}_e$ and torque $\bm \tau_e$ as a dynamical system and learning them from data. Specifically, the proposed framework utilizes lifted linear system derived from Koopman operator theory to capture the external force/torque in an explicit way.

\subsection{Lifted Linear System}
The Koopman operator embeds nonlinear dynamical systems into equivalent infinite-dimensional linear representations from data \cite{Mamakoukas2023}. This provides a data-driven pattern to capture unknown nonlinear dynamics such as external force/torque of (\ref{quadrotordynamics}) in an explicit way. Thus, our work adopts this property to learn $\bm f_e$ and $\bm \tau_e$ in a data-driven manner. The basic idea of lifted linear system is introduced briefly as follows.

Consider an unknown nonlinear dynamics with control input
\begin{equation}
	\dot{\bm{x}} = \bm{f}(\bm{x},\bm{u}) \label{nonlinear_input}
\end{equation}
where $\bm{x}\in \mathbb{R}^n, \bm{u} \in \mathbb{R}^p$. We define a set of scalar valued functions $\bm g$, which are called embedding (lifting) functions. The set of embedding functions defines an infinite dimensional Hilbert space $\mathcal{H}$. We assume that the embedding functions take the form
\begin{equation}
	\bm{g}(\bm{x},\bm{u}) = \bm{\Phi}(\bm{x})+\bm{L}\bm{u}
\end{equation}
where ${\bm{L}}$ is a constant matrix. Under the assumption that control input $\bm u$ does not evolve in the Hilbert space $\mathcal{H}$, the dynamics of (\ref{nonlinear_input}) can be reformulated as \cite{Proctor2018,Morton2019}
\begin{equation}
		\begin{aligned}
			\bm{\mathcal{K}} \bm{g}(\bm{x}(t_0),\bm{u}(t_0)) &= \bm{g}(\bm{x}(t_0+t_s),0) = \bm{\Phi}(\bm{x}(t_0+t_s)) \\&=
			\left[ \begin{matrix}
				\bm{A}&		\bm{B}\\
			\end{matrix} \right] \left[ \begin{array}{c}
				\bm{\Phi}\left( \bm{x} \right)\\
				\bm{\bm{u}}\\
			\end{array} \right]
		\end{aligned}
\end{equation}
where $\bm{x}(t_0+t_s)=\bm{x}(t_0)+\int_{t_0}^{t_0+t_s}\bm{f}(\bm{x,u}){\text{dt}}$ with sampling time $t_s$. The above equation figures the forward-time evolution of the embedding functions $\bm g(\bm x)$. Taking the notion $\bm z(t) \triangleq \bm{\Phi}(\bm x(t))$, the above equation can be rewritten as 
\begin{equation}
	\bm{z}(t_0+t_s) = \bm{A}\bm{z}(t_0)+\bm{B}\bm{u}(t_0) \label{lls}
\end{equation}
which is called \textbf{lifted linear system (LLS)} of unknown nonlinear dynamics (\ref{nonlinear_input}). If there is a sequence of data $\{(\bm x_0,\bm u_0),\cdots,(\bm x_{T-2},\bm u_{T-2}),(\bm x_{T-1})\}$ of unknown nonlinear dynamics (\ref{nonlinear_input}), the matrices $\bm{A}$ and $\bm{B}$ can be approximated by solving the following least-squares minimization problem
\begin{equation}
	\bm{A},\bm{B} = \mathop{\arg\min\limits_{\bm{A},\bm{B}}} \Vert \bm{z}_{1:T-1} - (\bm{A}\bm{z}_{0:T-2}+\bm{B}\bm{u}_{0:T-2}) \label{LS_AB} \Vert
\end{equation}
where $\bm{z}_{0:T-2}= [ \bm z_0,\bm z_1,\bm z_i,\cdots,\bm z_{T-2} ]^\mathrm{T}$,  $\bm{z}_{1:T-1}= [\bm z_1,\bm z_2,\bm z_i,\cdots,\bm z_{T-1}]^\mathrm{T}$, $\bm{u}_{0:T-2}= [\bm u_0,\bm u_1,\bm u_i,\cdots,\bm u_{T-2}]^\mathrm{T}$, $\bm z_i$ represents the embedding state corresponding to the $i$-th sample $\bm{x}_i\ (i=0,1,\cdots,T-1)$. When the sequence length $T$ of $\bm{z}_{0:T-2}, \bm{u}_{0:T-2}$ is finite, the above least-squares minimization problem leads to a finite dimensional approximation of Koopman operator \cite{Hao2024}. The equation (\ref{lls}) is in discrete form as it is derived in a data-driven manner. Thus, according to (\ref{lls}), the continuous finite-dimensional approximation of unknown nonlinear system (\ref{nonlinear_input}) could be written as 
\begin{equation}
	\left\{ \begin{array}{c}
		\dot{\bm{z}}=\bm{A_c}\bm{z}+\bm{B_c}\bm{u}\\
		\bm x = \bm{\Phi}^{-1}(\bm \Phi(\bm x))\\
	\end{array} \right. \label{lift_linear}
\end{equation}
where $\bm{A_c}=\log(\bm{A})/t_s$,  $\bm{B_c}=\bm{B}/(\int_{0}^{t_s}e^{\bm{A_c}t}\text{dt})$, $\bm{z} = \bm{\Phi}(\bm{x}):\mathbb{R}^n \rightarrow \mathbb{R}^K$, $K\gg n$, $\bm{u}\in \mathbb{R}^p$, $\bm{A_c} \in \mathbb{R}^{K\times K}$,  $\bm{B_c} \in \mathbb{R}^{K\times p}$, $t_s$ is the sampling interval. 

\subsection{A Learning-Based Scheme to Capture External Force and Torque}
In this work, the unknown dynamics of external force and torque are assumed to be described by \cite{OConnell2022}
\begin{equation}
	\dot{\bm \chi} = \bm{\xi}(\bm \chi,\bm \zeta) \label{dynamics}
\end{equation}
where $\bm \chi = [ \bm f_e, \bm \tau_e]^\mathrm{T}$, $\bm{\xi}$ represents the unknown dynamics of the external force and torque,  and $\bm{\zeta}$ denotes the \textit{control input} of (\ref{dynamics}) and it is usually the state or substate (part of full state) of quadrotor. The selection of $\bm{\zeta}$ will be discussed in the implementation section later.

As discussed above, the unknown nonlinear dynamics like (\ref{dynamics}) can be captured by a finite-dimensional LLS (\ref{lift_linear}) from data. While it is not an easy task to find the embedding functions $\bm{\Phi}$. Therefore, a deep neural network parameterized by weights
 $\bm{\theta}=W^1,...,W^{L+1}$ is raised to approximate the finite dimensional embedding functions  
\begin{equation}
	\bm{\Phi}(\bm \chi;\bm{\theta}) = W^{L+1}\phi(W^L(...\phi(W^1 \bm \chi...)) \label{dnn}
\end{equation}
where $\phi$ is the ReLU activation function. Sample data from tethered-UAV system flying with a baseline controller to construct a dataset $\{\bm{x}^i_{0:m-1},\bm{u}^i_{0:m-1}\}$ $(i=1,\cdots,s)$ which contains $s$ trajectories. Every trajectory in the dataset has $m$ steps data, $\bm{x}^i_{0:m-1}=[\bm{x}^i_0,\cdots,\bm{x}^i_{m-1}]^\mathrm{T}$, $\bm{u}^i_{0:m-1}=[\bm{u}^i_0,\cdots,\bm{u}^i_{m-1}]^\mathrm{T}$. The labels for offline training $\bm{\chi}^i_{0:m-1}=[\bm{\chi}^i_{0},\cdots,\bm{\chi}^i_{m-1}]^\mathrm{T}$ are calculated by using (\ref{quadrotordynamics}).

To find the embedding functions associated with Koopman operator, we minimize the following loss function
\begin{equation}
	L =  \beta_1 L_{\text{forward}} + \beta_2 L_{\text{backward}} + \beta_3 L_{\text{recons}} \label{loss_function}
\end{equation}
where $\beta_1$, $\beta_2$ and $\beta_3$ are positive hyperparameters. The detailed implementation of learning LLS is depicted in Algorithm \ref{learn_lls}. Three loss functions $L_{\text{forward}}$, $L_{\text{backward}}$, and $L_{\text{recons}}$ are developed as follows. 

\begin{figure}[!t]
	\renewcommand{\algorithmicrequire}{\textbf{Input}}
	\renewcommand{\algorithmicensure}{\textbf{Output}}
	\begin{algorithm}[H]
		\caption{Offline Learning LLS}
		\label{learn_lls}
		\begin{algorithmic}[1]
			\REQUIRE A dataset $\{\bm{x}^i_{0:m-1},\bm{u}^i_{0:m-1}\}$ $(i=1,\cdots,s)$.
			\STATE Calculate labels $\bm{\chi}^i_{0:m-1}$ for training.
			\WHILE {$L$ not decrease to tolerance}
			\STATE Calculate embedding states $\bm{z}_{0:m-1}^i$,  $\bm{z}_{0:m-2}^i$, $\bm{z}_{1:m-1}^i$ with $\bm{\Phi}(\bm \chi;\bm{\theta})$.
			\STATE Estimate matrices $\bm{A}^i$, $\bm{B}^i$ by solving $\min\limits_{\bm{A}^i,\bm{B}^i} \Vert \bm{z}_{1:m-1}^i - (\bm{A}^i\bm{z}_{0:m-2}^i+\bm{B}^i\bm{\zeta}_{0:m-2}^i) \Vert$.
			\STATE Predict embedding states by iterating LLS forward and backward in time, $\leftidx{^{\text{f}}}{\hat{\bm{z}}}{^i_{0:m-1}}$, $\leftidx{^{\text{b}}}{\hat{\bm{z}}}{^i_{0:m-1}}$.
			\STATE Update parameters of $\bm{\Phi}(\bm \chi;\bm{\theta})$ and matrix $\bm{C}$ by minimizing loss function $L$ (\ref{loss_function}).
			\ENDWHILE
			\ENSURE Embedding functions $\bm{\Phi}(\bm \chi;\bm{\theta})$ and matrix $\bm{C}$.
		\end{algorithmic} 
	\end{algorithm} 
\end{figure}

\textbf{\romannumeral1) Multi-Steps Prediction Loss $L_{\text{forward}}$ and $L_{\text{backward}}$}: The embedding state at each sampling instant can be represented as $\bm{z}^i_{0:m-1} = [\bm{z}_0^i,\cdots,\bm{z}_{m-1}^i]^\mathrm{T}=[\bm{\Phi}(\bm{\chi_{0}}^i),\cdots,\bm{\Phi}(\bm{\chi}_{m-1}^i)]^\mathrm{T}$. The state predicted by iterating LLS (\ref{lift_linear}) with the initial state $\bm{z}_0^i$ forward in time and backward in time \footnote{Backward dynamics: $\bm{z_k}=\bm A^{-1}(\bm{z}_{k+1}-\bm B\bm{\zeta}_k)$} are denoted as $\leftidx{^\text{f}\hat{\bm{z}}_{0:m-1}^i} = [\leftidx{^\text{f}\hat{\bm{z}}_0^i},\cdots,\leftidx{^\text{f}\hat{\bm{z}}_{m-1}^i}]^\mathrm{T}$ and $\leftidx{^\text{b}\hat{\bm{z}}_{0:m-1}^i} = [\leftidx{^\text{b}{\hat{\bm{z}}_0^i}},\cdots,\leftidx{^\text{b}\hat{\bm{z}}_{m-1}^i}]^\mathrm{T}$, respectively. The loss function of the forward and backward prediction error is raised as 
\begin{equation}
	\begin{aligned}
		L_{\text{forward}} = \sum_{i=1}^{s} \mu_1^{i} {\rm MSE}(\bm{z}_{0:m-1}^i,\leftidx{^\text{f}\hat{\bm{z}}_{0:m-1}^i})  \\
		L_{\text{backward}} = \sum_{i=1}^{s} \mu_2^{i} {\rm MSE}(\bm{z}_{0:m-1}^i,\leftidx{^\text{b}\hat{\bm{z}}_{0:m-1}^i})\label{L2}
	\end{aligned}
\end{equation}
where the $\mu_1,\mu_2 \in(0,1)$ are hyperparameters. Loss function (\ref{L2}) focuses on minimizing the multi-steps prediction error forward and backward in time, which is beneficial to predict state more precisely over a longer horizon.

\textbf{\romannumeral2) Reconstruction Loss $L_{\text{recons}}$}: The state of nonlinear system should be constructed from the invariant Hilbert space $\mathcal{H}$. The second equation in (\ref{lift_linear}) needs to obtain the inverse expression of embedding functions (\ref{dnn}), which is relatively challenging. Therefore, this work utilizes $\bm{\chi}=\bm C\bm z$ for the reconstruction. The loss functions $L_{\text{recons}}$ is given by
\begin{equation}
	L_{\text{recons}}=\sum_{i=1}^{s}\Vert \bm{\chi}_{0:m-1}^i - \bm{C} \bm{z}_{0:m-1}^i \Vert
\end{equation}
where matrix $\bm{C}$ is parameterized by a linear layer without bias.

\subsection{Theoretical Guarantee for the Learned Dynamics}
As the learned dynamics may have a divergent prediction, which will engender an inaccurate prediction of external force/torque. Thus, a constraint for bounding the prediction error of the learned dynamics must be considered.
\begin{theorem}\label{error_boundness}
	Assuming that $\alpha_{\bm{\chi}} >0 $ and $\alpha_{\bm{\zeta}} >0 $ such that $\Vert \bm{\chi}(t_0+t_s)-\bm{\chi}(t_0)\Vert \le \alpha_{\bm{\chi}}$  and $\Vert \bm{\zeta}(t_0+t_s) - \bm{\zeta}(t_0) \Vert \le \alpha_{\bm{\zeta}}$ are satisfied, the approximated Koopman operator $\hat{\bm{\mathcal{K}}}$ is stable and the embedding functions associated it have a Lipschitz constant $L_{\bm{\Phi}}$, then the prediction error is globally bounded.
\end{theorem}
\begin{proof}
	The global prediction error \cite{Mamakoukas2023} induced by the approximated Koopman operator $\hat{\bm{\mathcal{K}}}$ after $n$ sampling instants is given by
	\begin{equation}
		\bm{E_n} = \bm{z}(t_0+nt_s) - \hat{\bm{\mathcal{K}}}^n\bm{z}(t_0)
	\end{equation}
	
	By recursively iterating (\ref{lls}), the global prediction error can be represented as
	\begin{equation}
		\begin{aligned}
			\Vert \bm{E_n} \Vert &= \Vert \bm{\Phi}(t_0+nt_s) -\hat{\bm{\mathcal{K}}}^n\bm{\Phi}(t_0) \Vert  \\
			&=\Vert \sum_{i=0}^{n-1}\hat{\bm{\mathcal{K}}}^i\bm{e}(t_0+(n-i)t_s)\Vert \\
			&\le \sum_{i=0}^{n-1}\Vert \hat{\bm{\mathcal{K}}}^i\Vert \cdot \Vert \bm{e}(t_0+(n-i)t_s) \Vert
		\end{aligned}
	\end{equation}
	where the $\bm{e}(t)=\bm{\hat{\chi}}(t)-\bm{\chi}(t)$ is the local prediction error and $\bm{\hat{\chi}}$ denotes the prediction of external force and torque. 
	
	Since the $\bm{\chi},\bm{\zeta}$ and the embedding functions $\bm{\Phi}$ are Lipschitz continuous, the local prediction error $\bm{e}(t)$ is bounded by (Theorem 1 in \cite{Hao2024})
	\begin{equation}
		\begin{aligned}
			\lim_{n_h\rightarrow \infty} \text{sup} \Vert \bm{e}(t) \Vert &=(\Vert \bm{CA} \Vert L_{\bm{\Phi}}+1) \alpha_{\bm{\chi}}+\Vert \bm{CB}\Vert \alpha_{\bm{\zeta}} \\ &+\max_{\overline{\bm{\chi}}\in \mathbb{B}} \Vert \overline{\bm{\chi}} -\bm{C}\bm{\Phi(\overline{\bm{\chi}}}) \Vert \triangleq c
		\end{aligned}
	\end{equation}
	where $n_h$ is number of the last hidden layer of the embedding functions and $\mathbb{B}$ denotes a set that contains $\bm \chi$ at every sampling instant. As the prediction local error $\bm e(t)$ is bounded and the approximated Koopman operator $\hat{\bm{\mathcal{K}}}$ is stable, then 
	\begin{equation}
		\Vert \bm{E_n} \Vert \le \sum_{i=0}^{n-1}\Vert \hat{\bm{\mathcal{K}}}^i\Vert \cdot \Vert \bm{e}(t_0+(n-i)t_s) \Vert\le \Vert c \Vert \sum_{i=0}^{n-1}\Vert \bm{\mathcal{K}}^i \Vert
	\end{equation}
	is bounded. Therefore the prediction error is bounded.
\end{proof}

\subsection{Constrain Lipschitz Constant of Embedding Functions}
As stated in Theorem \ref{error_boundness}, the prediction error is bounded if embedding functions $\bm \Phi$ have a Lipschitz constant. To achieve this, the spectral normalization (SN) is utilized to constrain the Lipschitz constant of embedding functions $\bm{\Phi}$ in this work. By definition, the Lipschitz constant of a function $\rho$ is equal to the maximum spectral norm of its gradient i.e. $\Vert \rho \Vert_{\text{Lip}} = \sup \sigma (\nabla  \rho)$. Therefore, for embedding functions $\bm \Phi$ that parameterized by the deep neural network (\ref{dnn}), the Lipschitz constant can be bounded by constraining the spectral norm of each layer. 

The Lipschitz constant of a linear layer $g(\bm{x})=\bm{W}\bm{x}$ is given by $\Vert g \Vert_{\text{Lip}} = \sup \sigma(\nabla(g))=\sup \sigma(\bm{W})=\sigma(\bm{W})$. Thus, the Lipschitz constant of the embedding functions (\ref{dnn}) can be computed as \footnote{Using the inequality $\Vert g_1 \circ g_2 \Vert_{\text{Lip}} \le \Vert g_1 \Vert_{\text{Lip}}\cdot \Vert g_2 \Vert_{\text{Lip}}$.} 
\begin{equation}
	\begin{aligned}
		\Vert \bm{\Phi}(\bm{\chi}) \Vert_{\text{Lip}} &\le \Vert g^{L+1} \Vert_{\text{Lip}} \cdot \Vert \phi_L \Vert_{\text{Lip}} \cdots \Vert \phi_1 \Vert_{\text{Lip}} \cdot \Vert g_1 \Vert_{\text{Lip}} \\
		&= \prod_{l=1}^{L+1}\sigma(\bm{W}^l)
	\end{aligned}
\end{equation}
where the Lipschitz constant of ReLU activation function is equal to 1, i.e., $\Vert \phi_l \Vert_{\text{Lip}} = 1$. The spectral normalization applied to weights of each layer during training is given by
\begin{equation}
	\hat{\bm{W}} = \bm{W} / \sigma(\bm{W}) \cdot \gamma ^{\frac{1}{L+1}}  \label{sn}
\end{equation}

\begin{lemma} \label{lemma_lip}
	With the realization of spectral normalization (\ref{sn}) to the embedding functions, the Lipschitz constant of the embedding functions will satisfy:
	\begin{equation}
		\Vert \bm{\Phi}(\bm{\chi}) \Vert_{\text{Lip}}\le \gamma
	\end{equation}
	where $\gamma$ is the expected Lipschitz constant of $\bm{\Phi}$.
\end{lemma}
\begin{proof}
	Applying the spectral normalization to each layer of DNN, the Lipschitz constant of the embedding functions satisfies
	\begin{equation}
		\Vert \bm{\Phi}(\bm{\chi}) \Vert_{\text{Lip}} \le \prod_{l=1}^{L+1}\sigma(\hat{\bm{W}}^l) = \prod_{l=1}^{L+1}\gamma^{\frac{1}{L+1}}=\gamma
	\end{equation}
	So the Lipschitz constant of embedding functions could be bounded with spectral normalization. 
\end{proof}

\section{Implementation of Controller with Learned Dynamics}
As a well-known controller, robust MPC takes uncertainty into consideration \cite{Rawlings2017}. However, these methods rely on some hypotheses about uncertainty, which is hard to satisfy in practice. In addition, the closed-loop performance of MPC and accuracy of model are inextricably linked. The Neural Predictor can capture dynamics that is absent in the nominal dynamics from data. Therefore, the Neural Predictor integrated into MPC framework leads our proposed NP-MPC framework for tethered-UAV system, which is shown in Fig. \ref{illustration_framework} .

The MPC adopts an optimization-based approach to generate a sequence of optimal control inputs in a receding horizon manner. Thus, the optimization problem of NP-MPC considering the learned dynamics can be written as follows
\begin{equation}
	\begin{aligned} \label{nmpc}
			&\operatorname*{minimize}_{\mathbf{\bm{\bar{u}}}}&& \sum_{i=0}^{N-1}\left(\bm e_{\bm x_i}^\mathrm{T}\bm{Q}\bm e_{\bm x_i}+\bm e_{\bm u_i}^\mathrm{T}\bm{R}\bm e_{\bm u_i}\right)+\bm e_{\bm x_N}^\mathrm{T}\bm{P}\bm e_{\bm x_N}  \\
			&\text{subject to}&& \begin{aligned}&\bm{\bar{x}}_{0}=\bm{x}_k,  \quad  \bm{\bar{x}}_{N}\in\mathcal{X}_f,\\
				&\bm{\bar{x}}_{i+1}=\bm{f}_{\text{nominal}}(\bm{\bar{x}}_i,\bm{\bar{u}}_i)+\bm{\Xi}\bm{\hat{\chi}}_i,\end{aligned}  \\
			&&& \bm{\hat{\chi}}_i=\bm{C}\bm{z}_i=\bm{C}(\bm{A}^k\bm{z}_{i-1}+\bm{B}^k\bm{\zeta}_{i-1}),
			\\
			&&&\bm{\bar{x}}_i\in\mathcal{X},\quad \bm{\bar{u}}_i\in\mathcal{U},\quad \forall i=0,\ldots,N-1 
	\end{aligned}
\end{equation}
where $\bm{{x}}_k$ denotes the state of quadrotor at time instant $k$; $\bm{\bar{x}}_i$ and $\bm{\bar{u}}_i$ are the predicted state and control input at time instant {$k+i$}; $\bm e_{\bm x_i}=\bm{\bar{x}_i}-\bm x^*_{k+i}$ and $\bm e_{\bm u_i}=\bm{\bar{u}_{i}}-\bm u^*_{k+i}$ denote the state and control tracking errors at time time instant $k+i$; $\bm{\hat{\chi}}_i$ represents predicted external force/torque at time instant $k+i$; $\bm{\Xi}$ is a matrix mapping the learned external force/torque into state space; $\bm{z}_{-1}=\bm{\Phi}(\bm{\chi}_{k-1})$, $\bm{\zeta}_{-1}=\bm{\zeta}_{k-1}$, $\bm{\zeta}_i$ $(i=0,\cdots,N-1)$ is constructed with the predicted state $\bm{\bar{x}}_i$; the reference states and control inputs are generated by a trajectory generator; $N$ is the prediction horizon; $\mathcal{X}$, $\mathcal{U}$ and $\mathcal{X}_f$ are the state, control input and terminal state constraint sets; $\bm{Q}$ and $\bm{R}$ are the state and control input weighting matrices, $\bm{P}$ is the terminal state weighting matrix; $\bm{f}_{\text{nominal}}$ is the nominal discrete-time dynamics of quadrotor. At time instant $k$, utilizing the dataset $\{\bm{x}^{k}_{0:T-1},\bm{u}^{k}_{0:T-1}\}$, as formalized in Algorithm \ref{learn_lls}, the matrices $\bm{A}^k$ and $\bm{B}^k$ are estimated by solving the problem $\min\limits_{\bm{A}^k,\bm{B}^k} \Vert \bm{z}_{1:T-1}^{k} - (\bm{A}^k\bm{z}_{0:T-2}^{k}+\bm{B}^k\bm{\zeta}_{0:T-2}^{k}) \Vert$, where $\bm{x}_{0:T-1}^{k}=[\bm{x}_{k-T},\cdots,\bm{x}_{k-1}]^\mathrm{T}$, $\bm{u}_{0:T-1}^{k}=[\bm{u}_{k-T},\cdots,\bm{u}_{k-1}]^\mathrm{T}$.

\begin{figure*}[!t]
	\centering
	\includegraphics[width=36pc]{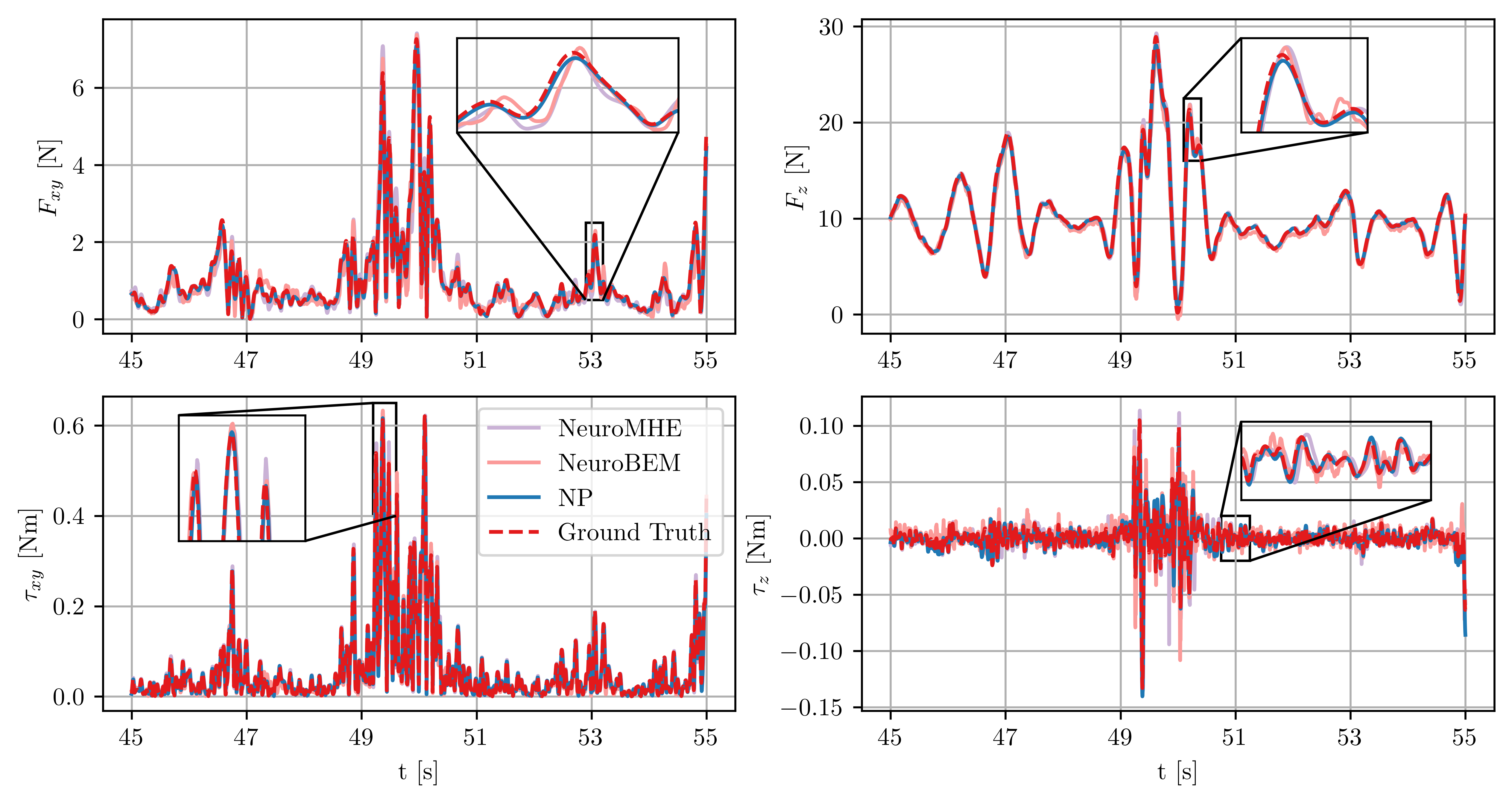} 
	\caption{Force and torque estimation performance of Neural Predictor, NeuroBEM and NeuroMHE on a part of ``Random points" test dataset (from 45s to 55s). The estimation errors on the whole ``Random points" trajectory are $\text{RMSE}_{F}^{\text{NeuroMHE}}=0.260$, $\text{RMSE}_{\tau}^{\text{NeuroMHE}}=0.012$, $\text{RMSE}_{F}^{\text{NeuroBEM}}=0.530$, $\text{RMSE}_{\tau}^{\text{NeuroBEM}}=0.013$, $\text{RMSE}_{F}^{\text{NP}}\textbf{=0.088}$, $\text{RMSE}_{\tau}^{\text{NP}}=\textbf{0.008}$, where $F_{xy} = \sqrt{{F_x}^2+{F_y}^2}$, $\tau_{xy} = \sqrt{{\tau_x}^2+{\tau_y}^2}$, $F = \sqrt{{F_x}^2+{F_y}^2+{F_z}^2}$ and $\tau = \sqrt{{\tau_x}^2+{\tau_y}^2+{\tau_z}^2}$. } \label{comparison_bem}
\end{figure*}

\section{Experiments and Results}
\label{s4}
In this section, we evaluate the performance of Neural Predictor both in numerical simulations and real-world flights to demonstrate superior effectiveness of our proposed scheme. The lifting functions $\bm \Phi$ are parameterized by a two-layer deep neural network with 128 neurons of each layer. The activation function is selected as ReLU. The hyperparameters are set to $\beta_1=\beta_2=\beta_3=1.0$, $\mu_1=\mu_2=0.999$. The dimension $K$ of LLS is set to 24 in all numerical simulations and physical experiments. The length of data sequence for matrices $\bm{A}$ and $\bm{B}$ estimation during training and online execution is set to 40.

\subsection{Numerical Evaluation}
In this part, Neural Predictor is compared with state-of-the-art learning-based estimator NeuroMHE \cite{Wang2024e} and NeuroBEM \cite{Bauersfeld2021} for demonstrating superiority of our proposed approach for capturing external force/torque. The dataset for training and validating in this part is the same as flight dataset\footnote{\url{https://rpg.ifi.uzh.ch/NeuroBEM.html}} presented in \cite{Bauersfeld2021}. This dataset was collected from a variety of agile and high-intensity flight maneuvers. The ground-truth of quadrotor states and time-varying external force/torque can be obtained from the dataset.

The Neural Predictor is trained using supervised learning as the NeuroMHE in numerical experiments (See Section VII-A in \cite{Wang2024e}) does. The inputs of Neural Predictor are linear and angular velocities of quadrotor. To ensure a fair comparison, the training dataset is selected as a 10-s long trajectory segment from a wobbly circle flight, covering a velocity range of 0.19 m/s to 5.18 m/s, which is consistent with training data used for NeuroMHE in \cite{Wang2024e}. And for test datasets, we also use the same 13 unseen agile flight trajectories in \cite{Wang2024e} to compare performance of Neural Predictor against NeuroMHE and NeuroBEM. 

\renewcommand\arraystretch{1.15}
\begin{table*}
	\caption{Comparison results of estimation errors (RMSEs) on 6 unseen agile flight trajectories}
	\label{table_bem}
	\begin{tabular}{m{2cm}<{\centering}m{1cm}<{\centering}m{1cm}<{\centering}m{1cm}<{\centering}m{1cm}<{\centering}m{1cm}<{\centering}m{1cm}<{\centering}m{1cm}<{\centering}m{1cm}<{\centering}m{1cm}<{\centering}m{1cm}<{\centering} m{1cm}<{\centering}}
		\Xhline{1.pt}
		\textbf{Flight Trajectory} & \textbf{Method} & \textbf{$\bm F_x$[N]} & \textbf{$\bm F_y$[N]} & \textbf{$\bm F_z$[N]} & \textbf{$\bm \tau_x$[Nm]} & \textbf{$\bm \tau_y$[Nm]} & \textbf{$\bm \tau_z$[Nm]} & \textbf{$\bm F_{xy}$[N]} & \textbf{$\bm \tau_{xy}$[Nm]} & \textbf{$\bm F$[N]} & \textbf{$\bm \tau$[Nm]} 
		\\ 
		\Xhline{1.pt}
		& NeuroBEM & 0.164 & 0.185 & 0.456 & 0.013 & 0.011 & 0.006 & 0.247 & 0.017 & 0.518 & 0.018 \\ 
		Linear oscillation & NeuroMHE &0.119 & 0.105 & 0.186 & 0.011 & 0.007 & 0.005 & 0.159 & 0.013 & 0.244 & 0.014 \\ 
		& NP & \textbf{0.041} & \textbf{0.024} & \textbf{0.101} & \textbf{0.007} & \textbf{0.004} & \textbf{0.004} & \textbf{0.048} & \textbf{0.008} & \textbf{0.112} & \textbf{0.009} \\ 
		\hline
		& NeuroBEM & 0.169 & 0.158 & 0.463 & 0.009 & 0.009 & 0.004 & 0.231 & 0.013 & 0.517 & 0.013 \\ 
		Race track\_1 & NeuroMHE & 0.141 & 0.092 & 0.115 & 0.007 & 0.004 & 0.004 & 0.168 & 0.009 & 0.204 & 0.009 \\ 
		& NP &\textbf{ 0.025} & \textbf{0.018} & \textbf{0.048} & \textbf{0.005} & \textbf{0.003} & \textbf{0.002} & \textbf{0.031} & \textbf{0.006} & \textbf{0.057} & \textbf{0.006} \\ 
		\hline
		& NeuroBEM & 0.110 & 0.129 & 0.470 & 0.006 & 0.009 & 0.004 & 0.170 & 0.011 & 0.499 & 0.011 \\ 
		3D circle\_2 & NeuroMHE & 0.140 & 0.135 & 0.075 & 0.003 & 0.002 & 0.004 & 0.194 & 0.004 & 0.208 & 0.006 \\ 
		& NP & \textbf{0.055} & \textbf{0.021} & 0.139 & 0.007 & 0.003 & 0.004 & \textbf{0.058} & 0.007 & \textbf{0.151} & 0.008 \\ 
		\hline
		& NeuroBEM & 0.145 & 0.168 & 0.584 & 0.010 & 0.012 & 0.006 & 0.221 & 0.015 & 0.624 & 0.017 \\ 
		Figure-8\_3 & NeuroMHE & 0.118 & 0.133 & 0.151 & 0.010 & 0.006 & 0.005 & 0.178 & 0.012 & 0.233 & 0.013 \\ 
		& NP & \textbf{0.058} & \textbf{0.029} & \textbf{0.127} & \textbf{0.007} & \textbf{0.004} & \textbf{0.004} & \textbf{0.065} & \textbf{0.008} & \textbf{0.143} & \textbf{0.009} \\ 
		\hline
		& NeuroBEM & 0.244 & 0.198 & 0.921 & 0.009 & 0.006 & 0.003 & 0.314 & 0.015 & 0.974 & 0.016 \\ 
		Melon\_2 & NeuroMHE & 0.254 & 0.213 & 0.094 & 0.005 & 0.003 & 0.004 & 0.331 & 0.005 & 0.344 & 0.007 \\ 
		& NP & \textbf{0.069} & \textbf{0.025} & 0.182 & 0.009 & 0.004 & 0.005 & \textbf{0.073} & 0.010 & \textbf{0.197} & 0.011 \\ 
		\hline
		& NeuroBEM & 0.161 & 0.183 & 0.471 & 0.008 & 0.008 & 0.005 & 0.244 & 0.012 & 0.530 & 0.013 \\ 
		Random points & NeuroMHE & 0.115 & 0.114 & 0.204 & 0.010 & 0.006 & 0.005 & 0.162 & 0.012 & 0.260 & 0.012 \\ 
		& NP & \textbf{0.032} &\textbf{ 0.028} & \textbf{0.077} & \textbf{0.007} & \textbf{0.004} & \textbf{0.003} & \textbf{0.042} & \textbf{0.008} & \textbf{0.088} & \textbf{0.008} \\ 
		\hline
		\Xhline{1.pt}
	\end{tabular}
	\label{MRFsum}
\end{table*}

The estimation performance of external force/torque on a part of aggressive trajectory (from 45s to 55s, labeled as ``Random points" in Table \ref{table_bem}) of NeuroMHE, NeuroBEM and Neural Predictor (NP) is shown in Fig. \ref{comparison_bem}. It is clear that all three methods can estimate the external force/torque accurately. However, during the aggressive maneuver of quadrotor (layout enlargement in Fig. \ref{comparison_bem}), both NeuroMHE and NeuroBEM fail to capture external force/torque well, resulting poor estimation performance. In contrast, the estimation of Neural Predictor remains closer to the ground truth, indicating that our proposed Neural Predictor maintains accurate force/torque estimation even in challenging scenarios. This can be attributed to the fact that Neural Predictor models the external force/torque as a linear dynamical system, with the state of quadrotor states serving as inputs. When quadrotor executes aggressive maneuvers, the changes of states drive Neural Predictor to respond quickly the changes of force/torque, which enables fast response. Furthermore, Neural Predictor, compared to state-of-the-art estimator NeuroMHE, has reduced the force and torque estimation error significantly by 66.15\% and 33.33\%, respectively. 

To further demonstrate the effectiveness of our approach, we evaluate the performance of Neural Predictor on 6 unseen agile flight trajectories. The detailed results are provided in Table \ref{table_bem}, where NP denotes our proposed Neural Predictor. The RMSEs of the estimation errors are computed in the same manner as NeuroMHE in \cite{Wang2024e}. It is clear that Neural Predictor achieves much smaller RMSEs for force/torque estimation and outperforms than NeuroMHE and NeuroBEM across all cases. Notably, our method has reduced the force estimation error by up to 72.06\% (for the``Race track\_1" trajectory in Table \ref{table_bem}).Regarding torque estimation, Neural Predictor outperforms NeuroMHE and NeuroBEM in four cases, with comparable performance in the remaining two cases. These results demonstrate that Neural Predictor owns good generalization ability across a variety of unseen aggressive trajectories and delivers superior performance for estimating force/torque compared to state-of-the-art learning-based estimator.

\begin{figure}[!t]
	\centering
	\includegraphics[width=21pc]{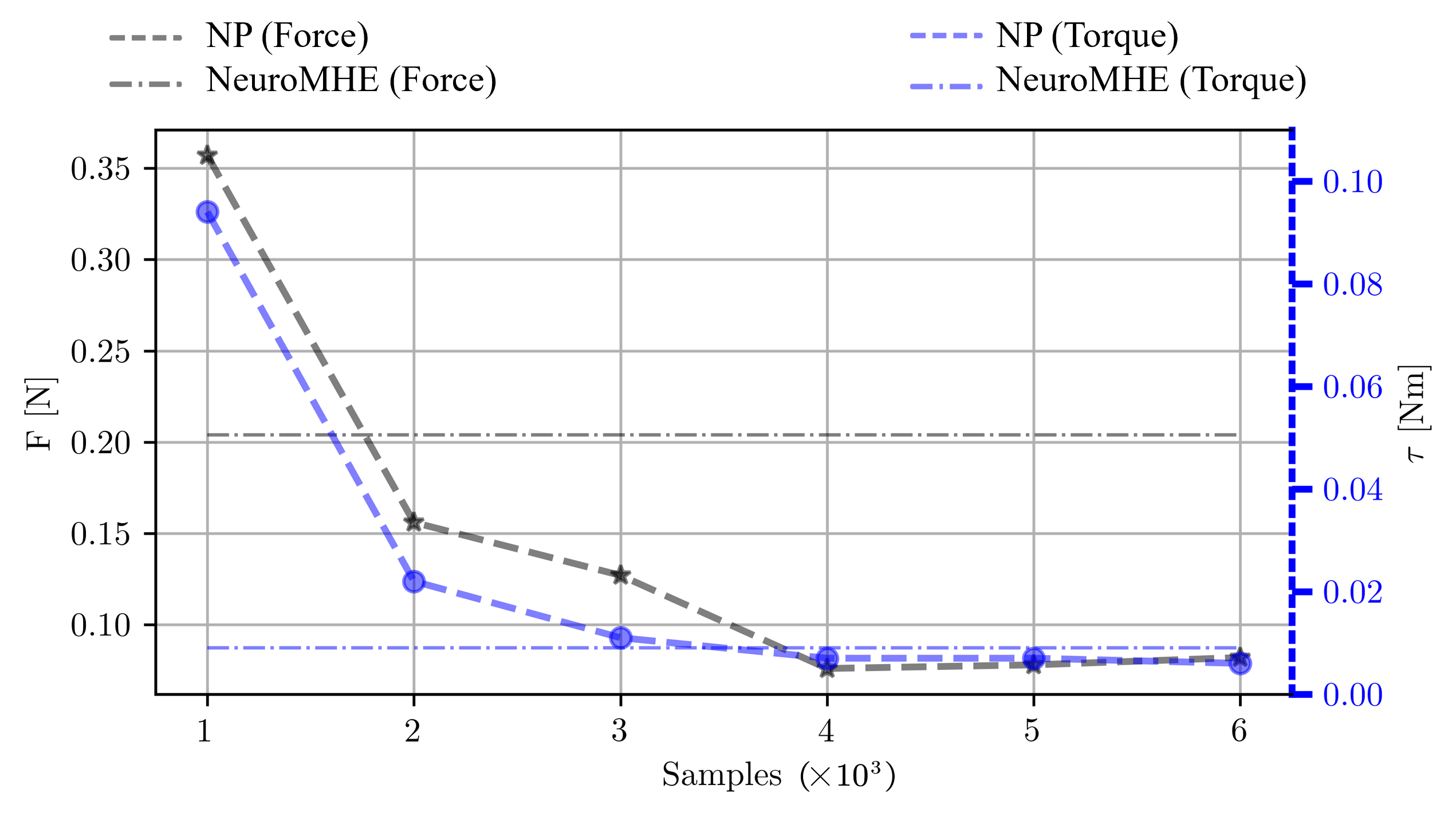} 
	\caption{The force and torque estimation RMSEs of Neural Predictor that is trained with different sample size on ``Race track\_1" test dataset. The gray and blue dash-dotted lines denote force and torque estimation RMSEs of NeuroMHE which is trained with 4K samples on ``Race track\_1" test dataset, respectively.} \label{sample_efficiency}
\end{figure}

To demonstrate the superior sample efficiency of the proposed method, we trained Neural Predictor with 1K, 2K, 3K, 4K, 5K and 6K samples\footnote{These samples represent 2.5-s, 5-s, 7.5-s, 10-s, 12.5-s and 15-s long segment of trajectory, respectively.} of training set. The estimation results on test dataset are presented in Fig. \ref{sample_efficiency}. As observed, the RMSE of the trained Neural Predictor decreases as the sample size increases. However, the performance of Neural Predictor trained with 4K, 5K, and 6K samples is nearly identical on the test dataset, indicating that 4K samples are sufficient to effectively capture the dynamics between force/torque and linear/angular velocity. Moreover, compared to NeuroMHE trained with 4K samples, Neural Predictor exhibits comparable performance even with fewer than 4K samples. Specifically, for force estimation, Neural Predictor achieves performance similar to that of NeuroMHE trained with 4K samples using only approximately 1.8K samples, significantly reducing the required sample size for training.

Overall, the numerical simulation results indicate that: 1) the proposed Neural Predictor owns capability of fast response and accurate estimation of external force/torque even there is an aggressive maneuvers, 2) the trained Neural Predictor demonstrates better generalization performance and sampling efficiency on unseen agile flight scenarios.

\subsection{Real-World Flight Experiments}
The setup for real-world flight experiments is shown in Fig. \ref{experiment_setup}. The custom quadrotor weighs 2.0 kg, with a motor wheelbase of 360 mm. It is equipped with a Pixhawk4 for attitude estimation and low-level control. The position and linear velocity of quadrotor are estimated from measurements of a motion capture system (VICON). The acceleration and angular velocity are measured from the accelerometers and gyroscope sensors of Pixhawk4. All implementations for solving the optimization problem (\ref{nmpc}) are deployed with CasADi \cite{Andersson2019} and acados \cite{Verschueren2022}. The data for training is collected at 50 Hz with a baseline MPC controller from a circular flight trajectory\footnote{The radius and flying speed are 1 m and 1.5 m/s, respectively.} of tethered-UAV system. The training set contains 60-s data (3000 data points), and the labels (external force/torque) for learning are calculated using equation (\ref{quadrotordynamics}). In the data collection experiment, the tether length is 0.8 m and the payload mass is 260 g, which are unknown to the pipeline of the proposed framework.

\begin{figure}[!t]
	\centering
	\includegraphics[width=18pc]{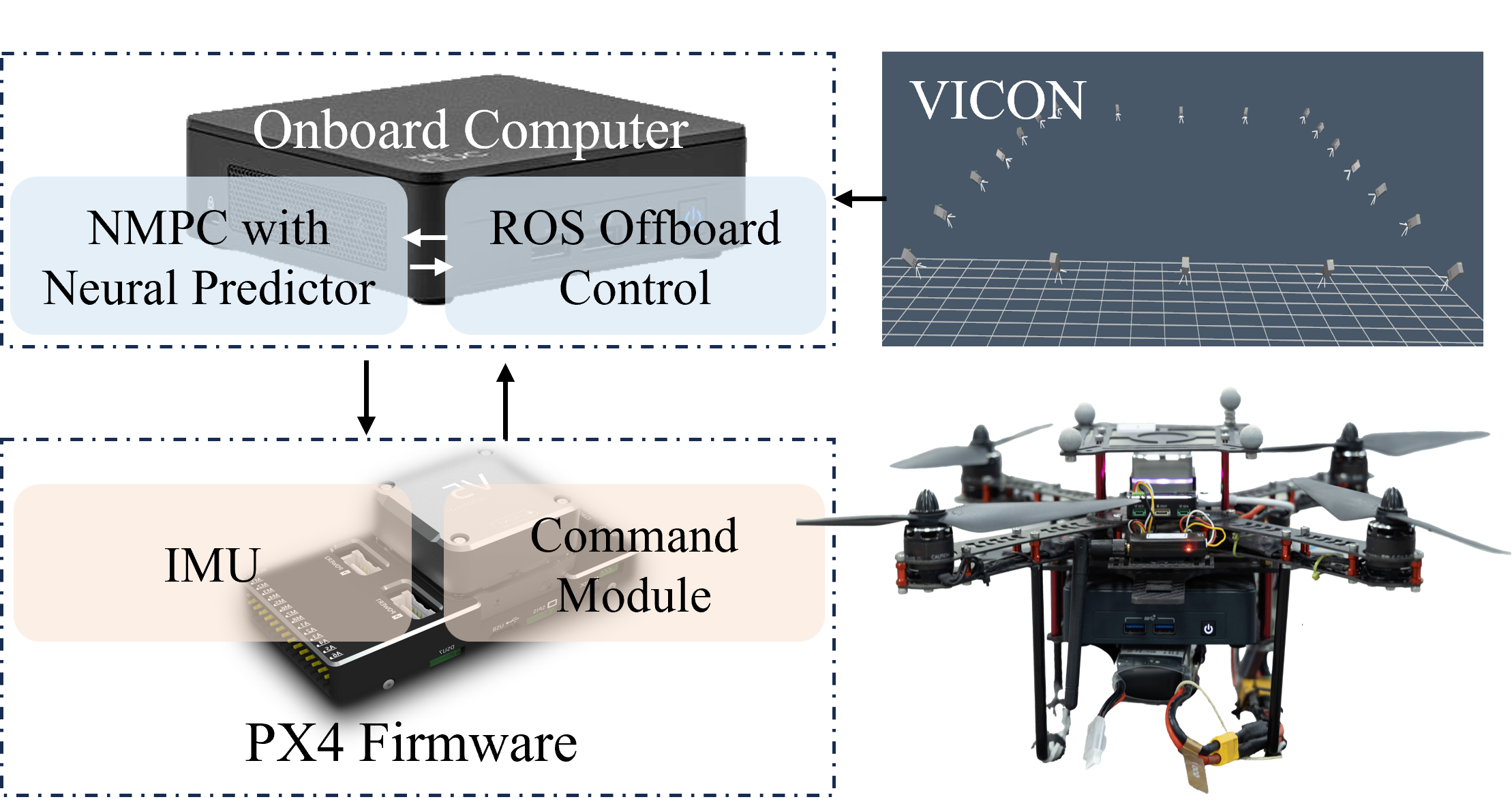} 
	\caption{The schematic of real-world flight experiments setup. } \label{experiment_setup}
\end{figure}

We first demonstrate the ability of Neural Predictor for accurately predicting the external force. During hovering with a payload weighted 260 g at a height of 1.6 m, an additional payload weighted 100 g is attached to the tethered-UAV. The prediction results of external force are shown in Fig. \ref{force}. It is seen that the dynamical response of Neural Predictor to changes in external force is almost instant. And when the weight of payload increased by 100 g, the output of Neural Predictor increased by 0.96 N with a standard deviation of 0.004 N. {Note that as there is residual dynamics, the changes of output of Neural Predictor does not strictly equal to the force of gravity on additional payload.} These results indicate that Neural Predictor achieves high-fidelity real-time estimation of external force. Furthermore, it can be adaptive to different unknown payload. 

Subsequently, real-world flight experiments involving trajectory tracking are conducted to verify the superiority of NP-MPC against nominal MPC and other two similar frameworks, GP-MPC \cite{Torrente2021} and  KNODE-MPC \cite{Chee2022}. For fair comparison,all models of three learning-based frameworks (NP-MPC, GP-MPC, and KNODE-MPC) are trained on identical 60-s circular trajectories obtained from the aforementioned data collection experiment. The KNODE network architecture maintains an identical configuration to the lifting function, while the GP model employs a radial basis function (RBF) kernel with 100 uniformly sampled data points for training. The inputs of NP, GP and KNODE models are all linear and angular velocity of quadrotor like we did in the numerical evaluation. The reference signals are produced by a polynomial trajectory generator and the reference height is 1.6 m, while the suspended payload mass is maintained at 260 g throughout the experiments. 

The performance results of four frameworks for tracking two trajectories lasting 50-s are summarized in Table \ref{table_xyz} and visualized in Fig. \ref{2dtraj}. Notably, due to the modeling mismatch induced by payload, nominal MPC has significant tracking errors in X-Y plane and Z-axis. In contrast, all three learning-based frameworks, GP-MPC, KNODE-MPC and NP-MPC, captured the external force induced by payload from data, which results smaller tracking error against nominal MPC. In the circle trajectory tracking experiment, GP-MPC, KNODE-MPC and NP-MPC own better performance than nominal MPC since they learned modeling mismatch from data and compensated it to the MPC framework. However, our proposed NP-MPC outperforms the other two frameworks and reduces the tracking error by  53.45\% in X-Y plane, 67.45\% in Z-axis compared with nominal MPC. Similar performance trends persist in the lemniscate trajectory tracking experiments, where all learning-based frameworks maintain advantages over nominal MPC. However, GP-MPC and KNODE-MPC exhibit substantial tracking degradation in high-curvature regions (characterized by small radius of curvature). Despite the lemniscate trajectory's geometric dissimilarity to the circular training data, NP-MPC maintains precise tracking even under these geometrically distinct conditions, achieving error reductions of 58.16\% (X-Y plane) and 76.72\% (Z-axis) compared to nominal MPC. These experimental findings collectively demonstrate that the Neural Predictor not only accurately estimates the external force induced by payload but also exhibits great generalization capabilities to out-of-distribution trajectory patterns unseen during training.

\begin{figure}[!t]
	\centering
	\includegraphics[width=21pc]{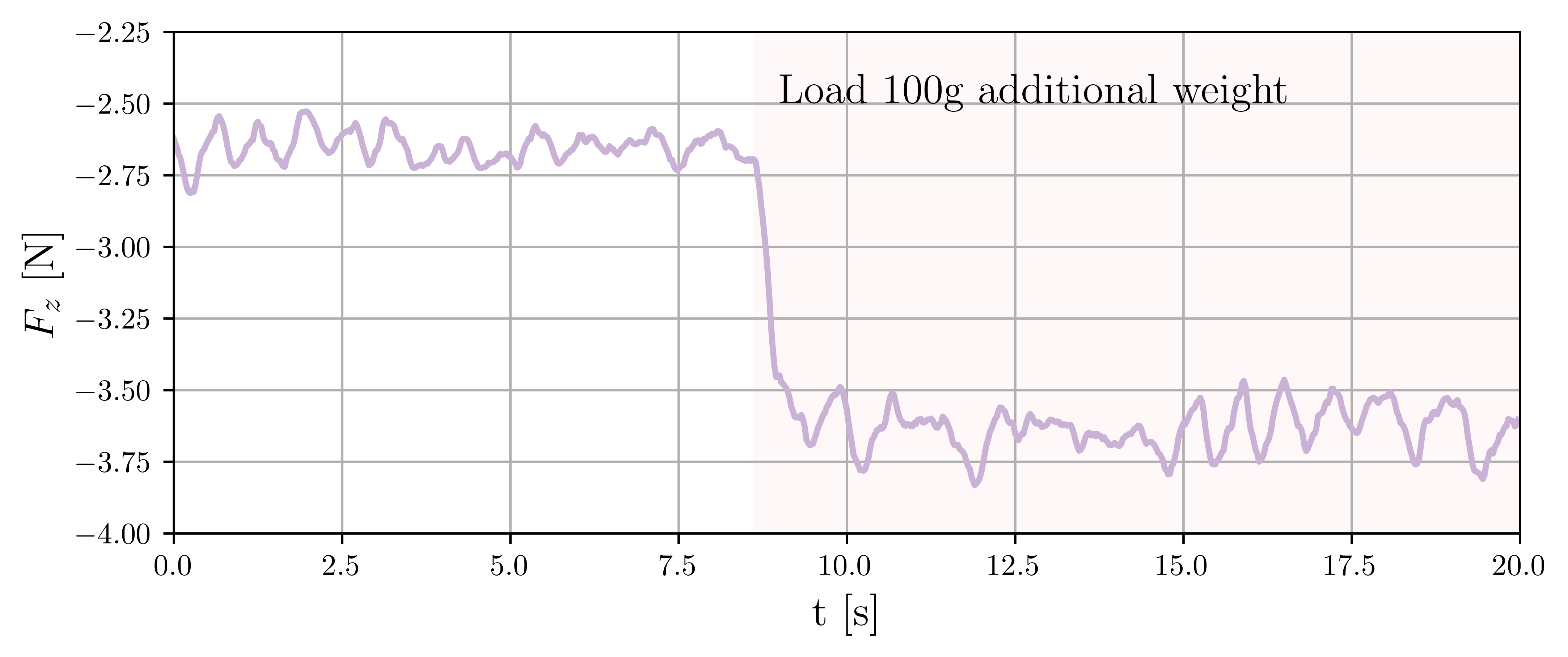} 
	\caption{Estimation of the external force in Z-axis when additional payload is attached.} \label{force}
\end{figure}

\renewcommand\arraystretch{1.15}
\begin{table}
	\caption{Comparison results of tracking errors (RMSEs) on two real-world flight trajectories}
	\label{table_xyz}
	\begin{tabular}{m{1.6cm}<{\centering}m{1.8cm}<{\centering}m{1.6cm}<{\centering}m{1.6cm}}
		\Xhline{1.pt}
		\textbf{Trajectory} & \textbf{Method} & \textbf{$\bm E_{xy}$ [m]} & \textbf{$\bm E_z$ [m]} 
		\\ 
		\Xhline{1.pt}
		& Nominal MPC & 0.1797 &0.2329 \\ 
		Circle & GP-MPC &0.1414 &0.1692 \\ 
		& KNODE-MPC & {0.0919} & {0.1307}  \\ 
		& NP-MPC & \textbf{0.0837} & \textbf{0.0758}  \\ 
		\hline
		& Nominal MPC &0.2029 & 0.2539  \\ 
		Lemniscate & GP-MPC & 0.1878 & 0.1966  \\ 
		& KNODE-MPC &{0.1503} & { 0.1406} \\ 
		& NP-MPC &\textbf{ 0.0849} & \textbf{0.0591} \\ 
		\hline
		\Xhline{1.pt}
	\end{tabular}
\end{table}

\begin{figure*}[!htbp]
	\centering
	\includegraphics[width=34pc]{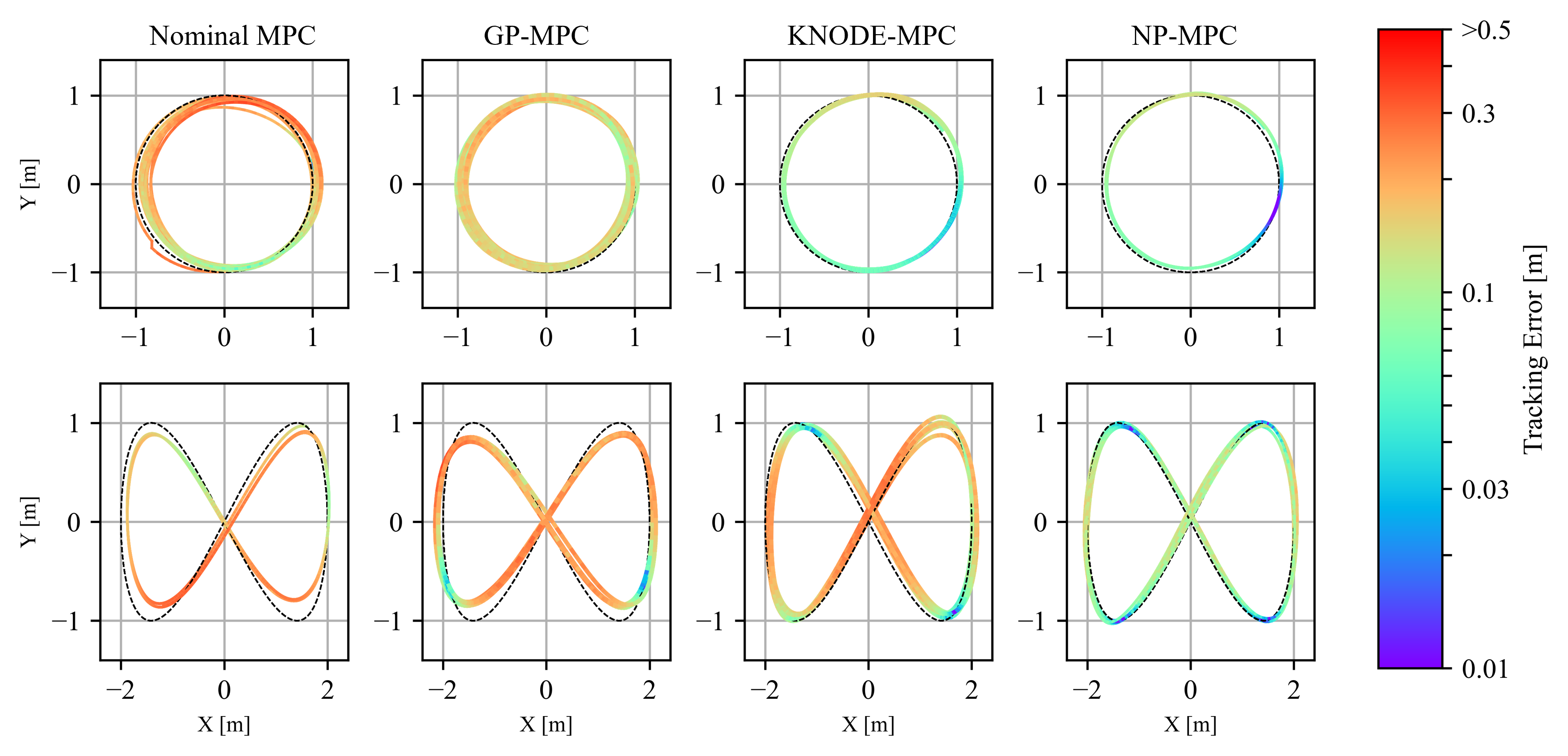} 
	\caption{Plane trajectory tracking performance of nominal MPC, GP-MPC, KNODE-MPC and NP-MPC with flying speed 1.5 m/s. The reference trajectory is shown in dot line. The tracking error refers the RSMEs between real flight trajectory and the reference trajectory.  }
	\label{2dtraj}
\end{figure*}

To validate the effectiveness of the proposed method for estimating external force and torque induced by the payload's oscillations, we conducted a hovering flight experiment. The detailed setup is consistent with the one described in Section VI-B. The target point of hovering is $(0,0,2\ \text{m})$. The payload is subjected to an externally imposed impulsive force (along X-axis) to deliberately induce a oscillation at the beginning of the experiment. We compare the hovering performance of nominal-MPC and NP-MPC under oscillatory payload conditions. The results are shown in Fig. \ref{hover_pose} and Fig. \ref{payload_pose}. The results indicate that the nominal-MPC can not effectively keep the tether-UAV system hovering at the target point under oscillatory payload conditions, with tracking errors exceeding 0.2 m ($\text{RMSE}=0.242$) in position. And there is a significant tracking error in Z-axis. While the proposed NP-MPC is caple of keeping the tether-UAV system hovering at the target point under oscillatory payload conditions, with tracking errors less than 0.1 m (RMSE=0.064) in position. The NP-MPC outperforms the nominal-MPC in terms of tracking performance, with a significant reduction in tracking errors. Beyond external force/torque estimation, our framework achieves oscillation mitigation (passive) through predictive compensation. The results in Fig. \ref{payload_pose} reveal a reduction in payload swing amplitude compared to nominal-MPC, with oscillations almost attenuated within 10 s. The nominal-MPC also achives oscillation mitigation, but the oscillation is eliminated relatively slower than NP-MPC. This owns to the fact that the nominal-MPC does not explicitly estimate the external force/torque induced by the payload's oscillations, leading to a slower response in oscillation mitigation and leading to a significant tracking error. The above results demonstrate that the proposed method can effectively estimate the external force/torque acting on the quadrotor, including in the presence of load oscillations. This capability is crucial for maintaining the quadrotor's stability and tracking performance during flight, particularly in scenarios where the payload is subject to oscillations. 

\begin{figure}[!t]
	\centering
	\includegraphics[width=17pc]{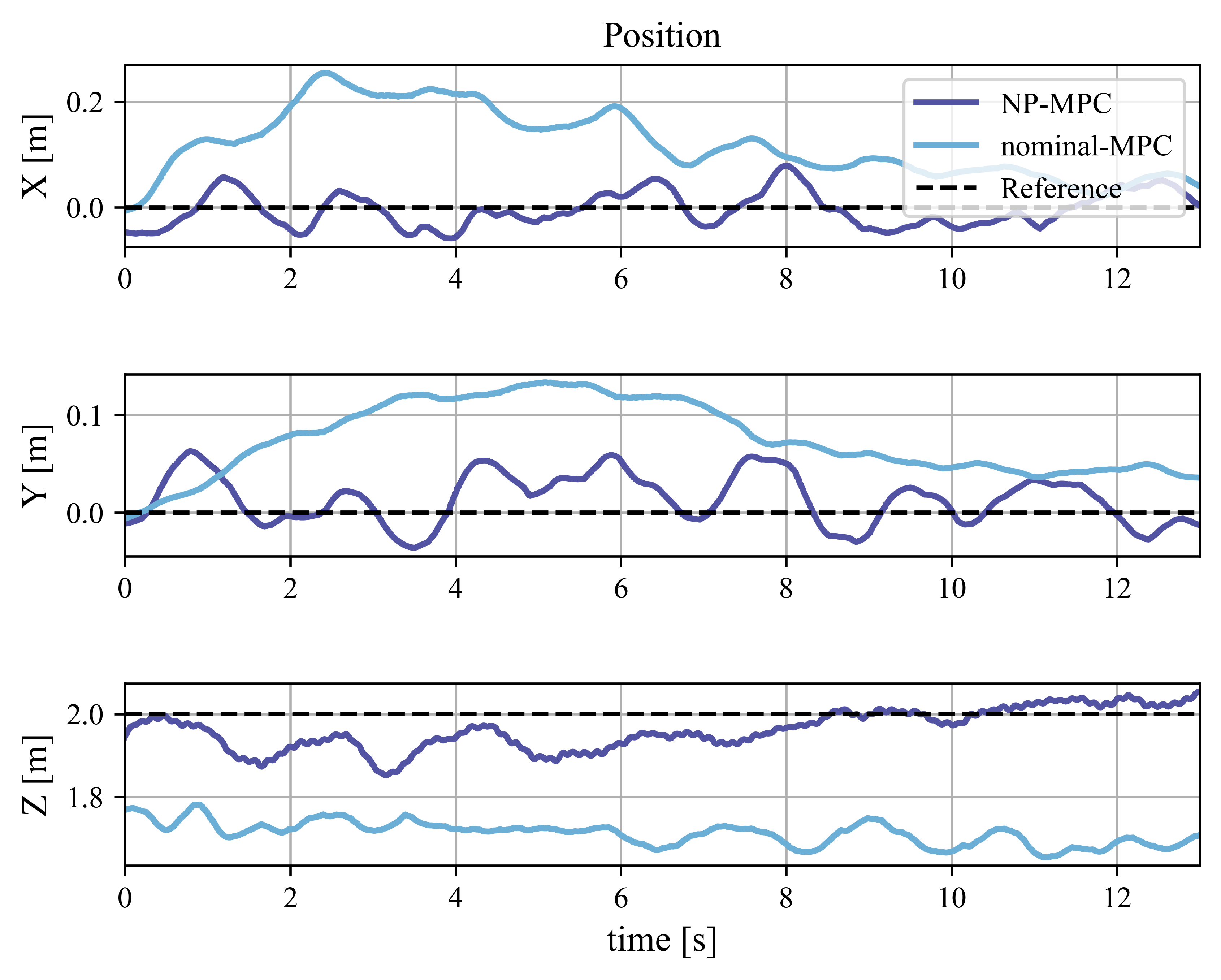} 
	\caption{The hovering experiment results of NP-MPC and nominal-MPC.} \label{hover_pose}
\end{figure}

\begin{figure}[!t]
	\centering
	\includegraphics[width=17pc]{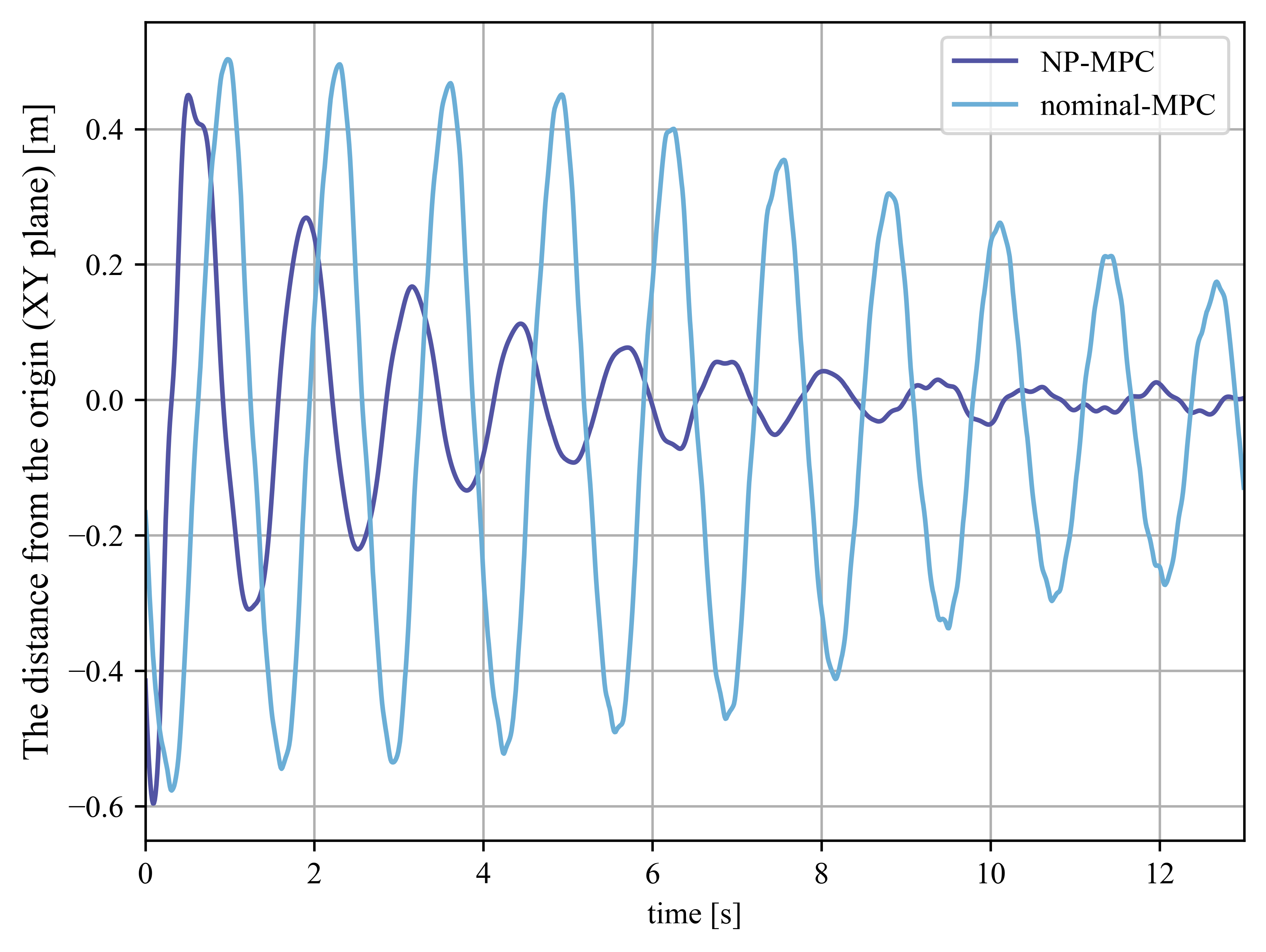} 
	\caption{The distance of payload from the origin (XY-axis plane).} \label{payload_pose}
\end{figure}

The computational time required for evaluating the neural network and solving the optimization problem (\ref{LS_AB}) on the companion computer\footnote{The companion computer of Pixhawk4 is equipped with an Intel Core i5-1135G7 and 16GB RAM} of Pixhawk4 is less than 5 ms. For solving the MPC problem (\ref{nmpc}), the computational time is less than 10 ms. This indicates that our framework is computationally efficient and feasible for practical applications.

Overall, the real-world flight results indicate that: 1) the fast response, accurate estimation and great generalization ability of proposed Neural Predictor have been further validated. 2) Neural Predictor cooperating with MPC can significantly improve the tracking performance and outperform other similar frameworks.

\section{CONCLUSIONS}
In this work, we present the Neural Predictor, a learning-based framework designed to accurately capture the external force/torque induced by suspended payload. Our critical insight is to model external force/torque as a dynamical system by utilizing DDKO theory. A theoretical guarantee is also given to bound prediction error. This makes Neural Predictor owns capability of fast response, good generalization, better sample efficiency and more accurate estimation of external force/torque over state-of-the-art estimator. Furthermore, combing Neural Predictor with MPC framework known as NP-MPC, significantly improves the closed-loop tracking performance and outperforms other similar frameworks in real-world experiments. 

The proposed method predicts the external force/torque induced by payload in a data-driven manner. The payload oscillation mitigation of the proposed method is achieved through predictive compensation, which is a passive approach. In future work, we plan to explore data-driven methods for actively controlling the payload's motion in conjunction with our proposed method. This will involve developing a more sophisticated control strategy that actively mitigates oscillations. We believe that this combination will enhance the overall performance of the tether-UAV system, particularly in dynamic scenarios where payload oscillations are a significant concern. Our future work will focus on extending the proposed method to address challenges in active payload oscillation mitigation and multi-UAVs cooperative transportation under unknown disturbances.

\bibliographystyle{IEEEtran}
\bibliography{main.bib}

\end{document}